\documentclass[twoside]{article}

\usepackage[accepted]{aistats2021}
\usepackage[round]{natbib}

\usepackage{url}            
\usepackage{booktabs}       
\usepackage{nicefrac}       
\usepackage{microtype}      
\usepackage{amsmath,amssymb,amsthm,amsfonts,dsfont,mathtools,latexsym,natbib}
\usepackage{enumerate}
\usepackage{xcolor}
\usepackage{hyperref}\hypersetup{
	colorlinks=true,
	linkcolor=blue,
	filecolor=magenta,      
	urlcolor=cyan,
	citecolor=teal,
} 
\usepackage{tikz}
\usepackage{algorithm}
\usepackage[noend]{algpseudocode}
\newcommand{\argmax}{\mathrm{argmax}}
\newcommand{\poly}{\mathrm{poly}}

\def\cE{\mathcal{E}}

\def\cC{\mathcal{C}}
\def\cO{\mathcal{O}}
\def\hV{\widehat{V}}
\def\hQ{\widehat{Q}}

\newcommand{\bN}{\mathbb{N}}

\newcommand{\abs}[1]{\left|#1\right|}
\newcommand{\expect}[1]{\mathbb{E}\!\left[#1\right]}
\newcommand{\prob}{\mathbb{P}}
\newcommand{\indict}[1]{\mathbb{I}\!\left[#1\right]}
\newcommand{\clip}[2]{\mathrm{clip}\!\left[#1\Big| #2\right]}
\newcommand{\algo}{\mathsf{Alg}}
\newcommand{\states}{\mathcal{S}}
\newcommand{\trans}{P}

\newcommand{\actions}{\mathcal{A}}

\newcommand{\mdps}{\mathcal{M}}
\newcommand{\gapmin}{\Delta_{\min}}
\newcommand{\gap}{\Delta}
\newcommand{\regret}{\mathrm{Regret}}
\newcommand{\discount}{\left(1-\gamma\right)}
\newcommand{\gaph}[1]{\Delta_{#1}}

\newcommand{\xah}[1]{(x_{#1},a_{#1})}
\newcommand{\xahk}[2]{(x_{#1}^{#2},a_{#1}^{#2})}

\newcommand{\Qxah}[1]{Q_{#1}\!(x_{#1},a_{#1})}
\newcommand{\Qxahk}[2]{Q_{#1}^{#2}\!(x_{#1}^{#2},a_{#1}^{#2})}

\newcommand{\Pxahk}[2]{P_{#1}\!\left(\cdot|x_{#1}^{#2},a_{#1}^{#2}\right)}
\newcommand{\rxah}[1]{r_{#1}\!(x_{#1},a_{#1})}
\newcommand{\Vxh}[1]{V_{#1}\!(x_{#1})}
\newcommand{\Vxhk}[2]{V_{#1}^{#2}\!(x_{#1}^{#2})}
\newcommand{\Nxah}[1]{N_{#1}\!(x_{#1},a_{#1})}
\newcommand{\Nxahk}[2]{N_{#1}^{#2}\!(x_{#1}^{#2},a_{#1}^{#2})}

\newcommand{\Qxa}[1]{Q\!(x_{#1},a_{#1})}
\newcommand{\rxa}[1]{r\!(x_{#1},a_{#1})}
\newcommand{\Nxa}[1]{N\!(x_{#1},a_{#1})}
\newcommand{\Vhatx}[1]{\widehat{V}\!(x_{#1})}
\newcommand{\Qhatxa}[1]{\widehat{Q}\!(x_{#1},a_{#1})}

\newcommand{\circled}[1]{\lower.7ex\hbox{\tikz\draw (0pt, 0pt)%
		circle (.5em) node {\makebox[1em][c]{\small #1}};}}

\newtheorem{theorem}{Theorem}[section]

\newtheorem{lemma}[theorem]{Lemma}

\newtheorem{defn}{Definition}[section]
\newtheorem{definition}{Definition}[section]

\allowdisplaybreaks

\begin{document}

%

%

\twocolumn[

\aistatstitle{$Q$-learning with Logarithmic Regret}

\aistatsauthor{ Kunhe Yang \And Lin F. Yang \And  Simon S. Du}

\aistatsaddress{ \texttt{ykh17@mails.tsinghua.edu.cn}\\
Tsinghua University \And  
\texttt{linyang@ee.ucla.edu}\\
University of California, Los Angeles \And 
\texttt{ssdu@cs.washington.edu}\\
University of Washington }]

\begin{abstract}
This paper presents the first non-asymptotic result showing a model-free algorithm can achieve logarithmic cumulative regret for episodic tabular reinforcement learning if there exists a strictly positive sub-optimality gap. We prove that the optimistic $Q$-learning studied in [Jin et al. 2018] enjoys a ${\mathcal{O}}\!\left(\frac{SA\cdot \mathrm{poly}\left(H\right)}{\Delta_{\min}}\log\left(SAT\right)\right)$ cumulative regret bound where $S$ is the number of states, $A$ is the number of actions, $H$ is the planning horizon, $T$ is the total number of steps, and $\Delta_{\min}$ is the minimum sub-optimality gap of the optimal Q-function. This bound matches the  lower bound in terms of $S,A,T$ up to a  $\log\left(SA\right)$ factor. We further extend our analysis to the discounted setting and obtain a similar logarithmic cumulative regret bound.
\end{abstract}

\section{Introduction}
$Q$-learning~\citep{watkins1992q} is one of the most popular classes of methods for solving reinforcement learning (RL) problems. 
$Q$-learning tries to estimate the optimal state-action value function ($Q$-function).
With a $Q$-function, at every state, one can just greedily choose the action with the largest $Q$ value to interact with the RL environment.
Compared to another popular class of methods, model-based learning, $Q$-learning algorithms (or more generally, model-free algorithms) often enjoy better memory and time efficiency\footnote{See Section~\ref{sec:pre} for the precise definitions of model-free and model-based algorithms in the tabular setting.}.
These are the main reasons why $Q$-learning is applied in solving a wide range of RL problems~\citep{mnih2015human}.

While model-free methods are widely applied in practice, most theoretical works study model-based RL.
In one of the most fundamental RL frameworks, tabular RL, which is the focus of this paper, the majority of works study model-based algorithms~\citep{kearns1999finite, kakade2003sample, singh1994upper, azar2013minimax,azar2017minimax,dann2015sample,dann2017unifying,agarwal2019optimality,max2019nonasymptotic} with a few exceptions \citep{strehl2006pac,jin2018qlearning,dong2019qlearning,zhang2020optimal}.
From a regret minimization point of view, the state-of-the-art analysis demonstrates that one can achieve a $\sqrt{T}$-type regret bound where $T$ is the number of episodes.
Although these bounds are sharp in the worst-case scenario, they do not reveal the favorable
structures of the environment, which can significantly decrease the regret.

One such structure is the existence of a strictly positive sub-optimality gap, i.e., for every state, there is a strictly positive value gap between the optimal action(s) and the rest (cf. Definition~\ref{defn:gap}).
In practice, arguably, nearly all environments with finite action sets satisfy some sub-optimality gap conditions. 
In Atari-games, e.g., Freeway, the optimal action has a value that is usually very distinctive from the rest of actions.
In many other environments with finite number of actions, e.g. those control environments in OpenAI gym \citep{1606.01540},  the gap condition usually holds.
Similar gap conditions can be observed in other environments (see e.g. \cite{kakade2003sample}).

Theoretically, the sub-optimality gap is extensively investigated in the bandit problems, which can be viewed as RL problems with the planning horizon being $1$. 
With this structure, one can drastically decrease the $\sqrt{T}$-type regret to $\log T$-type regret~\citep{bubeck2012regret,lattimore2018bandit,slivkins2019introduction}.
For RL, most existing works that can leverage this structure require additional assumptions about the environment, such as finite hitting time and ergodicity~\citep{jaksch2010near,tewari2007reinforcement,ok2018exploration} or access to a generator~\citep{zanette2019almost}.\footnote{The simulator allows the user to query any state-action pair.}
Recently, \citet{max2019nonasymptotic} presented a systematic study of episodic tabular RL with the gap structure.
They presented a novel algorithm which achieves the near-optimal $\sqrt{T}$-type regret in the worst scenario and $\log T$-type regret if there exists a strictly positive sub-optimality gap.
Furthermore, they also provided instance-dependent lower bounds for a class of reasonable algorithms.
See Section~\ref{sec:rel} for more detailed discussions.

However, to our knowledge, all existing works that obtain $\log T$-type regret bounds are about model-based algorithms.
It remains open whether model-free algorithms such as $Q$-learning can achieve $\log T$-type regret bounds.
Indeed, this is a challenging task.
As discussed in \cite{max2019nonasymptotic}, their analysis framework cannot be applied to model-free algorithms directly.
Later in this section, we also provide some technical explanations on why their approach is difficult to adopt.

\paragraph{Our Contributions}
We answer the aforementioned open problem by proving that the optimistic $Q$-learning algorithm studied in \cite{jin2018qlearning} enjoys $\cO\!\left(\frac{SA H^6}{\gapmin}\log \left(SAT\right)\right)$ cumulative regret where $S$ is the number states, $A$ is the number of actions, $H$ is the planning horizon and $\gapmin$ is the minimum sub-optimality gap.
To our knowledge, this is the first result showing model-free algorithms can achieve $\log T$-type regret.
Furthermore, our bound matches the lower bound by~\cite{max2019nonasymptotic} in terms of $S$, $A$ and $T$ up to a $\log\left(SA\right)$ factor.
Importantly, the algorithm does not need to know $\gapmin$.

Second, we extend our analysis to the infinite-horizon discounted setting with the regret defined in \cite{liu2020regret}, for which we show the optimistic $Q$-learning achieves $\cO\!\left(\frac{SA}{\gapmin\left(1-\gamma\right)^6}\log\left(\frac{SAT}{\gapmin \left(1-\gamma\right)}\right)\right)$ regret where $0 \!<\! \gamma\! <\! 1$ is the discount factor.

\paragraph{Main Challenges}
\label{sec:challenge}

Here we explain the main challenges of using existing analyses and give an  overview of our main techniques at a high level.
The existing proof in \cite{jin2018qlearning} bounds the regret in terms of a weighted sum of the estimation error of $Q$-function.
Note the estimation error scales 
$1/\sqrt{T}$ which in turn gives a $\sqrt{T}$-type regret, but cannot give a $\log T$-type regret bound.

For model-based algorithms, \citet{max2019nonasymptotic} introduced a novel notion, \emph{optimistic surplus} (cf. Equation~\eqref{eqn:opt_surplus}),
 which can be bounded by the estimation error of the transition probability.
The logarithmic regret bound can be proved via a clipping trick on top of the optimistic surplus.

Unfortunately, as acknowledged by \citet{max2019nonasymptotic}, their analysis is highly tailored to model-based algorithms.
First, model-free algorithms do not estimate the probability transition, so we cannot bound the optimistic surplus via this approach.
Secondly, although we can also obtain a formula for the optimistic surplus in each episode using the update rules of the $Q$-learning algorithm, the formula depends on the estimation error of $Q$-function in previous episodes.
This dependency makes it difficult to bound the optimistic surplus.
See Section~\ref{sec:opt_surplus} for more technical details.

\paragraph{Technique Overview}
In this paper, we adopt an entirely different \emph{counting} approach. We first write the total regret as expected sum over sub-optimality gaps appearing in the whole learning process, then use the estimation error of $Q$-function and the definition of sub-optimality gap to upper bound the number of times the algorithm takes suboptimal actions.

To obtain a sharp dependency on $\gapmin$, we divide the interval $[\gapmin,H]$ (the range of all gaps) into multiple subintervals.
We then bound the sum of learning error in each subinterval by its maximum value times the number of steps falling into this subinterval.
The number of steps in each layer is bounded through computing the weighted sum of  learning error across all the episodes $k\in[K]$.
See detailed discussion in Lemma~\ref{lemma:weighed-sum-learning-error} and Lemma~\ref{lemma:count-in-each-layer}.

\paragraph{Organization}
This paper is organized as follows.
In Section~\ref{sec:rel} we discuss related works.
In Section~\ref{sec:pre}, we introduce necessary definitions and backgrounds.
In Section~\ref{sec:main_results}, we present our main results and discussions.
In Section~\ref{sec:proof_sketch}, we give the proof of our theorem on the episodic setting.
We conclude in Section~\ref{sec:conclusion} and leave remaining proofs to the appendix.

\label{sec:intro}
\subsection{Related Work}
\label{sec:rel}
\paragraph{Gap-independent Finite-horizon and Infinite-horizon Discounted  RL}\footnote{There is another line of works on gap-independent infinite-horizon average-reward setting. This setting is beyond the scope of this paper.}
There is a long list of results about regret or sample complexity of tabular RL, dating back to \cite{singh1994upper}.
One line of works require access to a simulator where the agent can query samples freely from any state-action pair of the environment and therefore the agent does not need to design a strategy to explore the environment.~\citep{kearns1999finite, kakade2003sample, singh1994upper, azar2013minimax, sidford2018variance, sidford2018near,agarwal2019optimality,zanette2019almost,li2020breaking}.

Another line of works drop the simulator assumption and thus the agent needs to use advanced techniques, such as upper confidence bound (UCB)  to explore the state space~\citep{azar2017minimax,dann2015sample,dann2017unifying,dann2019policy,jin2018qlearning,strehl2006pac,zhang2020optimal,max2019nonasymptotic,zanette2019tighter,dong2019qlearning}.
In terms of the regret, the state-of-art result shows one can achieve $\widetilde{\cO}\!\left(\sqrt{SAH^2T} + \poly\left(S,A,H\right) \right)$ regret for which the first term nearly match the $\Omega\left(\sqrt{SAH^2T}\right)$up to logarithmic factors~\citep{dann2015sample,osband2016on}.\footnote{
	In this paper, we study the same setting as in \cite{jin2018qlearning} where the reward at each level is in $[0,1]$, and the transition probabilities at each level can be different.
	In another setting, the total reward is bounded by $1$ and the transition probabilities at each level are the same.
	The latter setting is more challenging to analyze and the worst-case sample complexity is still open~\citep{jiang2018open,wang2020long}.
	}
Among these results, only a few are  for model-free algorithms~\citep{strehl2006pac,jin2018qlearning,dong2019qlearning,zhang2020optimal} and only very recently, \citet{jin2018qlearning,zhang2020optimal} showed $Q$-learning can achieve $\sqrt{T}$-type regret bounds.

\paragraph{Sub-optimality Gap}
The results about gap-dependent regret bounds for MDP algorithms can be categorized into  asymptotic bounds and non-asymptotic bounds. 
Asymptotic bounds are only valid when the total number of steps $T$ is large enough. 
These bounds often suffer from the worst-case dependency on some problem-specific quantities, such as diameter
and worst-case hitting time.
Under the infinite-horizon average-reward setting, \citet{auer2007logarithmic} provided a logarithmic regret algorithm for irreducible MDPs. 
Besides dependency on hitting times, their regret also depends inversely on $\gap_*^2$, the squared distance between optimal and second-optimal policy.
Along this direction and improving over previous algorithm of \citet{burnetas1997optimal}, \citet{tewari2008optimistic} proposed an algorithm called Optimistic Linear Programming (OLP). OLP is proved to have $C(P)\log T$ regret asymptotically in $T$, where $C(P)$ depends on some diameter-related quantity as well as the sum over reciprocals of gaps for $(x,a)$ inside a critical set. 

For non-asymptotic bounds, \citet{jaksch2010near}  introduced \textsc{UCRL2} algorithm, which enjoys $\widetilde{\cO}\!\left(\frac{D^2S^2A}{\gap_*}\log T\right)$ regret where $D$ is the diameter. More recently, \cite{ok2018exploration} derived problem-specific lower bounds for both structured and unstructured MDPs. Their lower bound scales $SA\log T$ for unstructured MDP and $c\log T$ for structured MDP, where this $c$ depends on both the minimal action sub-optimality gap and the span of bias function, which can be bounded by diameter $D$.
For non-asymptotic bounds, \cite{max2019nonasymptotic} proved that model-based optimistic algorithm \textsf{StrongEuler} has gap-dependent regret bound that holds uniformly over $T$. Moreover, their bounds depend only on $H$ and not on any term such as hitting time or diameter.
In Section~\ref{sec:main_results}, we compare our result with the one in \cite{max2019nonasymptotic} in more detail.



\section{Preliminaries}
\label{sec:pre}
\paragraph{Episodic MDP}
An episodic Markov decision process (MDP) is a tuple $\mdps:=\left(\states,\actions,H,\trans,r\right)$, where $\states$ is the finite state space with $\abs{\states}=S$, $\actions$ is the finite action space with $\abs{\actions}=A$, $H\in\mathbb{Z}_+$ is the planning horizon, $\trans_h:\states\times\actions\to\Delta(\states)$ is the transition operator at step $h$ that takes a state-action pair and returns a distribution over states, and $r_h:\states\times\actions\to [0,1]$ is the deterministic reward function at step $h$. Each episode starts at an initial state $x_1\in\states$ picked by an adversary.

In this paper, we focus on deterministic policies.
A deterministic policy $\pi$ is a sequence of mappings $\pi_h:\states\to\actions$ for $h=1,\ldots,H$. 
Given a policy $\pi$, for a state $x \in \states$, the value  function of state $x\in\states$ at the $h$-step  is defined as 
\[
	V_h^{\pi}(x):=\expect{\sum_{h'=h}^{H} r_{h'}(x_{h'},\pi_{h'}(x_{h'}))
		\Bigg| 
		x_h=x},
\] and the associated $Q$-function of a state-action pair $(x,a)\in\states\times\actions$ at the $h$-step is 
\begin{align*}
	Q_h^{\pi}(x,a):=r_{h}&(x,a)\\
	+&\expect{\sum_{h'=h+1}^{H} r_{h'}(x_{h'},\pi_{h'}(x_{h'}))
		\Bigg| 
		{x_h=x\atop a_h=a}}.
\end{align*}
We let $\pi^*$ be the optimal policy such that $V^{\pi^*}(x)=V^*(x)=\argmax_{\pi}V^{\pi}(x)$ and $Q^{\pi^*}(x,a)=Q^*(x,a)=\argmax_{\pi}Q^{\pi}(x,a)$ for every $(x,a)$.
For episodic MDP, the agent interacts with the MDP for $K \in \mathbb{Z}^+$ episodes. 
For each episode $k = 1,\ldots,K$, the learning algorithm $\algo$ specifies a policy $\pi^k$, plays $\pi^k$ for $H$ steps and observes trajectory $(x_1,a_1),\cdots,(x_H,a_H)$. 
The total number of steps is $T=KH$, and the total regret of an execution instance of $\algo$ is then \[
 \regret(K)=\sum_{k=1}^{K}\left(V_1^*-V_1^{\pi^k}\right)\!(x_1^k).
\]
In this paper we focus on bounding the expected regret $\expect{\regret(K)}$ where the expectation is over the randomness from the environment.

\paragraph{Model-free Algorithm V.S. Model-based Algorithm}
In this paper we focus on \emph{model-free} $Q$-learning algorithms.
Formally, by model-free algorithms, we mean the space complexity of the algorithm scales at most \emph{linearly} in $S$ in contrast to the model-based algorithms whose space complexity often scales \emph{quadratically} with $S$~\citep{strehl2006pac,sutton1998reinforcement,jin2018qlearning}.
For episodic MDP, we will analyze the $Q$-learning with UCB-Hoeffding algorithm studied in \cite{jin2018qlearning} (cf. Algorithm~\ref{algo:ucb-q}).
At a high level, this algorithm maintains an upper bound of $Q^*$ for every $(s,a)$ pair and choose the action greedily at every episode.
The algorithm uses a carefully designed step size sequence $\{\alpha_k\}$ to update the upper bound based on the observed data.
\citet{jin2019provably} proved that Algorithm~\ref{algo:ucb-q} enjoys $\!\left(\sqrt{H^4SAT\log\left(SAT\right)}\right)$ regret, which is the first $\sqrt{T}$-type bound for model-free algorithms.

\begin{algorithm}[tb]
	\caption{{Q-learning with UCB-Hoeffding}\label{algo:ucb-q}}
	\begin{algorithmic}[1]
		\State \textbf{Initialize:} $Q_h\!(x,a)\gets H$ and $N_h\!(x,a)\gets0$ for all $(x,a,h)\in\states\times\actions\times[H]$.
		\State \textbf{Define} $\alpha_t=\frac{H+1}{H+t}$,  $\iota\gets\log\left(SAT^2\right)$.
		\For{episode $k \in [K]$}
		\State receive $x_1$.
		\For{step $h\in[H]$}
		\State Take action $a_h\gets\argmax_{a'\in\actions}Q_h\!\left(x_h,a'\right)$, observe $x_{h+1}$.
		\State $t=\Nxah{h}\gets \Nxah{h}+1$,
		\State $b_t\gets c\sqrt{H^3\iota/t}$, \Comment{$c$ is a constant that can be set to 4.}
		\State $\Qxah{h}\gets\left(1-\alpha_t\right)\Qxah{h}+\alpha_t\left[\rxah{h}+\Vxh{h+1}+b_t\right]$,
		\State $\Vxh{h}\gets\min\left\{H,\max_{a'\in \actions} Q_h\!\left(x_h,a'\right)\right\}$.
		\EndFor
		\EndFor
	\end{algorithmic}
\end{algorithm}
\paragraph{Sub-optimality Gap}
Our paper investigates what structures of the MDP enable us to improve the $\sqrt{T}$-type bound.
In this paper we focus on the positive sub-optimality gap condition~\citep{max2019nonasymptotic,du2019q,du2020agnostic}.
%
\begin{defn}[Sub-optimality Gap]
	\label{defn:gap}
Given $h \in [H]$, $(x,a) \in \states \times \actions$, the suboptimality gap of $(x,a)$ at level $h$ is defined as
$\gaph{h}\!(x,a):=V_h^*(x)-Q_h^*(x,a). $
\end{defn}
\begin{defn}Minimum Sub-optimality Gap]
	\label{asmp:gapmin_pos}
Denote by $\gapmin$ the minimum non-zero gap: $\gapmin:=\min_{h,x,a}\left\{\gaph{h}(x,a):\gaph{h}(x,a)\neq0\right\}$. 
\end{defn}
Note that if $\left\{\gaph{h}(x,a):\gaph{h}(x,a)\neq0\right\}=\emptyset$, then all the states are the same, and the MDP degenerates.
Otherwise we always have $\gapmin > 0$. For the rest of the paper, we focus on the case when $\gapmin > 0$.
In Section~\ref{sec:intro} we have discussed why many MDPs admit this structure.
Our main result  is a logarithmic regret bound of Algorithm~\ref{algo:ucb-q}.

\paragraph{Infinite-horizon Discounted MDP}
In this paper we also study infinite-horizon discounted MDP, which is a tuple $\mdps:=\left(\states,\actions,\gamma,\trans,r\right)$, where every step shares the same transition operator $\trans$ and reward function $r$. Here $\gamma$ denotes the discount factor, and there is no restart during the entire process. 
Let $\cC=\left\{\states\times\actions\times[0,1]\right\}^{*}\times\states$ be the set of all possible trajectories of any length. 
A non-stationary deterministic policy $\pi:\cC\to\actions$ is a mapping from paths to actions. The $V$ function and $Q$ function are defined as below ($c_i:=\left(x_1,a_1,r_1,\cdots,x_i\right)$).
\begin{align}
V^{\pi}(x)&:=\expect{\sum_{i=1}^{\infty}\gamma^{i-1}r(x_i,\pi(c_i))
	\Bigg| 
	x_1=x},\nonumber\\
Q^{\pi}(x,a)&:=\rxa{}+\expect{\sum_{i=2}^{\infty}\gamma^{i-1} r(x_{i},\pi(c_{i}))
	\Bigg| 
	{x_1\!=\!x\atop a_1\!=\!a}}.\nonumber
\end{align}
	Let $V^*(s)$ and $Q^*(s,a)$  denote respectively the value function and $Q$ function of the optimal policy $\pi^*$.
\begin{defn}[Sub-optimality Gap]
	\label{defn:discounted-gap}
	Given $(x,a) \in \states \times \actions$, the suboptimality gap of $(x,a)$ is defined as
	$\gaph{}\!(x,a):=V^*(x)-Q^*(x,a). $
\end{defn}
\begin{defn}[Minimum Sub-optimality Gap]
	\label{asmp:disounted-gapmin_pos}
	Denote by $\gapmin$ the minimum non-zero gap: $\gapmin:=\min_{x,a}\left\{\gaph{}(x,a):\gaph{}(x,a)\neq0\right\}$. 
\end{defn}

Again, if $\left\{\gaph{}(x,a):\gaph{}(x,a)\neq0\right\}=\emptyset$, all the states are the same, and the MDP degenerates.
Otherwise, we have $\gapmin > 0$.

Consider a game that starts at state $x_1$. A learning algorithm $\algo$ specifies an initial non-stationary policy $\pi_1$. 
At each time step $t$, the player takes action $\pi_t(x_t)$, observes $r_t$ and $x_{t+1}$, and updates $\pi_t$ to $\pi_{t+1}$.
The total regret of $\algo$ for the first $T$ steps is thus defined as $
\regret(T)=\sum_{t=1}^{T}\left(V^*-V^{\pi_t}\right)\!(x_t).
$
This definition was studied in \cite{liu2020regret}, which follows the sample complexity definition in \cite{kakade2003sample}.
For this setting, we study Algorithm~\ref{algo:infinite ucb-q}.
This is a simple adaptation of Algorithm~\ref{algo:ucb-q} that takes $\gamma$ into account, so we defer it to the appendix.
We prove Algorithm~\ref{algo:infinite ucb-q} also enjoys a logarithmic regret bound.

\section{Main Theoretical Results}
\label{sec:main_results}
Now we present our main results.

\paragraph{Main Result for Episodic MDP}
The following theorem characterizes the performance of Algorithm~\ref{algo:ucb-q} for episodic MDP. 
To our knowledge, this is the first theoretical result showing a model-free algorithm can achieve logarithmic regret of tabular RL.
\begin{theorem}[Logarithmic Regret Bound of $Q$-learning for Episodic MDP]
	\label{thm:episodic}
	The expected regret of Algorithm~\ref{algo:ucb-q} for episodic tabular MDP is upper bounded by
	 $\expect{\regret(K)}\le\cO\left(\frac{H^6SA}{\gapmin}
	\log\left({SAT}\right)\right)$.
\end{theorem}

An interesting advantage of our theorem is adaptivity.
Note the algorithm we analyze is exactly the same algorithm studied in \cite{jin2019provably}, which has been shown to achieve the worst-case $\sqrt{T}$-type regret bound.
Theorem~\ref{thm:episodic} suggests that one does not need to modify the algorithm to exploit the strictly positive minimum sub-optimality gap structure, Algorithm~\ref{algo:ucb-q} \emph{automatically adapts} to this benign structure. 
Importantly, Algorithm~\ref{algo:ucb-q} does not need to know $\gapmin$.

Proposition 2.2 in \cite{max2019nonasymptotic} suggested that any algorithm with sub-linear regret in the worst case, suffer an $\Omega\left(\sum_{(x,a), \gap_1(x,a) > 0}\frac{H^2}{\gap_1\left(x,a\right)} \log T\right)$ expected regret.
Therefore, the dependencies on $S$, $A$ and $T$ are nearly tight in Theorem~\ref{thm:episodic}.

One may wonder whether it is possible to obtain a regret bound that only depends the sum of positive gaps, e.g., $O\left(\sum_{(x,a), \gap_1(x,a) > 0}\frac{H^2}{\gap_1\left(x,a\right)} \log T\right)$, unlike ours, which is a multiple of $1/\gapmin$.
Unfortunately, \citet{max2019nonasymptotic} showed, all existing algorithms, including Algorithm~\ref{algo:ucb-q} and their algorithm, suffer an $\Omega\left(\frac{S}{\gapmin}\right)$ regret, and new algorithmic ideas are needed in order to circumvent this lower bound.



We compare Theorem~\ref{thm:episodic} with the regret bound for model-based algorithm in \cite{max2019nonasymptotic} (in big-$\cO$ form):\begin{align*}
\Bigg(\!&\!\sum_{\substack{(x,a): \\\exists h \in [H], \gap_{h}\!(x,a) >0 }}\!\frac{H^3}{\min_h \gap_h\left(x,a\right)} + \frac{SH^3}{\gap_{\min}} \\
&+ H^4SA\max\left(S,H\right) \log\!\left(\!\frac{SAH}{\gapmin}\!\right)\Bigg)\log\left(SAHT\right) 
\end{align*}

First recall our bound is for a model-free algorithm which is more space-efficient and time-efficient than the model-based algorithm in \cite{max2019nonasymptotic}.
In terms of the regret bound, Theorem~\ref{thm:episodic}'s dependency on $H$ is worse than that in their bound.
We remark that simple model-free algorithms may have a worse dependency on $H$ compared to model-based algorithms (e.g., see \cite{jin2018qlearning}).

%
%
%


Now let us consider an environment where there are $\sim SA$ state-action pairs whose gap is $\gapmin$.
Then the bound in \cite{max2019nonasymptotic} becomes \[
\left(\!\frac{H^3SA}{\gapmin} \!+\! H^4SA \max\!\left\{S,\!H\right\}\!\log\!\left(\!\frac{SAH}{\gapmin}\!\right)\!\right)\! \log\!\left(\!SAHT\!\right).
\]
In this regime, both Theorem~\ref{thm:episodic} and their bound have an $\frac{SA}{\gapmin}$ term. 
Their bound also has an additional $H^4 SA\max\left(H,S\right) \log\left(\frac{SAH}{\gapmin}\right)$ burn-in term which our bound does not have.
When $S$ is large compared to $H$ and $\gap_{\min}$, this term scales $S^2$ and can dominate other terms, so our bound is better.
The technical reason behind this phenomenon is that Algorithm~\ref{algo:ucb-q} uses the Hoeffding bound for constructing bonus on $Q$-value, which does not need burn-in.


\paragraph{Main Result for Infinite-horizon Discounted MDP}
Algorithm~\ref{algo:ucb-q} can be easily generalized to the discounted MDP.
See Algorithm~\ref{algo:infinite ucb-q} in the appendix.
We also obtain a logarithmic regret bound for infinite-horizon discounted MDP.

\begin{theorem}[Logarithmic Regret Bound of $Q$-learning for Infinite-horizon Discounted MDP]
	\label{thm:discounted}
	The expected regret of Algorithm~\ref{algo:infinite ucb-q} for infinite-horizon discounted MDP is upper bounded by 
	$\expect{\regret(T)}\le\cO\left(\frac{SA}{\gapmin\discount^6}
	\log\frac{SAT}{\gapmin\left(1-\gamma\right)}\right)$.
\end{theorem}
Theorem~\ref{thm:discounted} suggests that model-free algorithms can achieve logarithmic regret even in the infinite-horizon discounted MDP setting.
The main difference from Theorem~\ref{thm:episodic} is that $H$ is replaced by $\frac{1}{1-\gamma}$.
By analogy, we believe the dependencies on $S,A,T$ and $\gapmin$ are nearly tight and the dependency $\frac{1}{1-\gamma}$ can be improved.
The proof of Theorem~\ref{thm:discounted} is deferred to Appendix.

\section{Proof of Theorem~\ref{thm:episodic}}
\label{sec:proof_sketch}
In this section, we prove Theorem~\ref{thm:episodic}.

\paragraph{Notations}
	Let $Q_h^k(x,a),V_{h}^k(x),N_h^k(x,a)$ denote the value of $Q_h(x,a),V_{h}(x)$,and $N_h(x,a)$ right before the $k$-th episode, respectively.
Let $\indict{\cdot}$ denote the indicator function.
Let $\tau_h(x,a,i):=\max\left\{k:N_h^k(x,a)=i-1\right\}$ be the episode $k$ at which $\xahk{h}{k}=(x,a)$ for the $i$-th time. We will abbreviate $\Nxahk{h}{k}$ for $n_h^k$ when no confusion can arise. 
$\alpha_t^i$ is defined by the following: $\alpha_t\!=\!\frac{H+1}{H+t}$, $\alpha_t^0\!=\!\prod_{j=1}^{t}\!\left(1\!-\!\alpha_{j}\right)$ and $\alpha_t^i\!=\!\alpha_i\!\prod_{j=i+1}^{t}\!\left(1\!-\!\alpha_{j}\right)\ (i>0)$. Let $\beta_0=0$ and $\beta_t\!=\!4c\sqrt{\!\frac{H^3\iota}{t}}$ for $t\ge1$.

\paragraph{Proof of Theorem~\ref{thm:episodic}}
Our proof starts with the observation that the regret of each episode can be rewritten as the expected sum of sub-optimality gaps for each action:
\begin{align}
&\left(V_1^*-V_1^{\pi^k}\right)\!\left(x_1^k\right)\nonumber\\
=&\ V_1^*\!\left(x_1^k\right)
-Q_1^*\!\left(x_1^k,a_1^k\right)
+\left(Q_1^*-Q_1^{\pi^k}\right)\!\xahk{1}{k}\nonumber\\
=&\ \gaph{1}\left(x_1^k,a_1^k\right)
+\mathbb{E}_{s'\sim \Pxahk{1}{k}}
\left[\left(V_{2}^*-V_{2}^{\pi^k}\right)\!(s')\right]\nonumber\\
=&\ \cdots\ =\expect{\sum_{h=1}^H
	\gaph{h}\!\xahk{h}{k}\Bigg|
	a_h^k=\pi_h^k(x_h^k)}.
\label{eq:episodic-regret-decomp}
\end{align}

In order to bound $\gaph{h}\!\xahk{h}{k}$ by learning error $\left(\!Q_h^k-Q_h^*\right)\!\xahk{h}{k}$, we define the following concentration event.
\begin{definition}[Concentration of Learning Errors]
	\begin{align*}
	\cE_{\mathrm{conc}}&:=\Bigg\{
	\forall (x,\!a,\!h,\!k)\!
	:
	0\le \left(Q_h^k\!-\!Q_h^*\right)\!(x,a)\le\nonumber\\
	&\alpha_{n_h^k}^0 H\!+\!\sum_{i=1}^{n_h^k}\alpha_{n_h^k}^i\!\left(\!V_{\!h+1}^{\!\tau_h(x,a,i)}\!-\!V^*\!\right)\!(\!x_{h+1\!}^{\!\tau_h\!(x,a,i)}\!)\!+\!\beta_{\!n_h^k}\Bigg\} .
\end{align*}
\end{definition}
Intuitively, $\cE_{\mathrm{conc}}$ is the event in which all the learning errors of the value function is both bounded below (by zero) and bounded above. 

We now refer to \cite{jin2018qlearning} for the following lemma that 
shows $\cE_{\mathrm{conc}}$ happens with high probability
via a concentration argument.
\begin{lemma}[Concentration]
	\label{lemma: conc-event}
	 Event $\cE_{\mathrm{conc}}$ occurs
	w.p. at least $1-\nicefrac{1}{T}$.
\end{lemma}
Lemma~\ref{lemma: conc-event} suggests that Algorithm~\ref{algo:ucb-q} is optimistic on $\cE_{\mathrm{conc}}$. Combining with the greedy choice of actions yields
\begin{align}
V_h^{*}\left(x_h^k\right)\!=\!
Q_h^{*}\left(x_h^k,a^{\!*}\!\right)\!\le\! Q_h^{k}\left(x_h^k,a^{\!*}\!\right) \!\le\! Q_h^{k}\left(x_h^k,a_h^k\right).
\label{ineq:optimistic-bound-on-Vopt}
\end{align}
To bound $\gap_h(x_h^k,a_h^k)$, the following notion introduced in \cite{max2019nonasymptotic} is convenient.
If we define $\clip{x}{\delta}\!:=\!x\cdot\indict{x\ge\delta}$, then Ineq~\eqref{ineq:optimistic-bound-on-Vopt}  suggests that $\gaph{h}\!\xahk{h}{k}$ can be bounded by clipped estimation error:
\begin{align}
\gaph{h}\!\xahk{h}{k}&=\clip{V_h^{*}\!\left(x_h^k\right)-Q_h^*\!\xahk{h}{k}}{\gapmin}\nonumber\\
&\le\clip{\left(Q_h^k-Q_h^*\right)\!\xahk{h}{k}}{\gapmin}.
\label{ineq:gap-bounded-by-learning-error}
\end{align}

Our main technique to get $1/\gapmin$ instead of $1/\gapmin^2$ regret bound is to classify gaps of state-action pairs into different intervals and count them separately. 
Note the gap can range from $\gapmin$ to $H$. 
Thus, we divide the interval $\left[\gapmin,H\right]$ into $N$ disjoint intervals:
 $\left[\gapmin,2\gapmin\right),\cdots,\left[2^{N-1}\gapmin,2^N\gapmin\right]$, where $N=\left\lceil \log_2\left(\nicefrac{H}{\gapmin}\right)\right\rceil$.

 Lemma~\ref{lemma:count-in-each-layer} below is our main technical lemma which upper bounds the number of steps Algorithm~\ref{algo:ucb-q} chooses a sub-optimal action whose suboptimality is in a certain interval.
\begin{lemma}[Bounded Number of Steps in Each Interval]
	\label{lemma:count-in-each-layer} Under $\cE_{\mathrm{conc}}$, we have
 	for every $n\in\left[N\right]$, 
	\begin{align*}
	C^{(n)} &\!:=\!\Bigg|\!\left\{\!(k,h):\ {\left(Q_h^k-Q_h^*\right)\!\xahk{h}{k}\!\in\atop\!\left[2^{n-1}\gapmin,2^n \gapmin\right)}\ \right\}\!\Bigg|\nonumber\\
	&\le\mathcal{O}\left(\frac{H^6SA\iota}{4^n \gapmin^2}
	\right),\qquad\text{where $\iota = \log\left(SAT^2\right)$.}
	\end{align*}
\end{lemma}
Before we give the proof for Lemma~\ref{lemma:count-in-each-layer}, we first show how to use Lemma~\ref{lemma:count-in-each-layer} to prove Theorem~\ref{thm:episodic}.
\paragraph{Proof of Theorem~\ref{thm:episodic}}
	Since the trajectories inside $\cE_{\mathrm{conc}}$ have bounded empirical regret, and 
	complementary event $\overline{\cE_{\mathrm{conc}}}$ happens with sufficiently low probability,
	\begin{align}
	&\expect{\regret(K)} =\expect{\sum_{k=1}^K\!\sum_{h=1}^H
		\gaph{h}\!\left(x_h^k,a_h^k\right)}\nonumber\\=\ 
	&\!\sum_{\text{traj}}\!
	\prob\!\left(\text{traj}\right)\cdot\!
	\sum_{k,h}\!
	\gaph{h}
		\left(\!x_h^k,a_h^k\big|\text{traj}\right)
		\label{tmp0}\\
	{\le}&
	\!\sum_{\text{traj}\in\cE_{\mathrm{conc}}}\!
	\prob\!\left(\text{traj}\right)\!\cdot\!
	\sum_{k,h}\!
	\clip{\!\left(\!Q_h^k\!-\!Q_h^*\!\right)\!
		(\!x_h^k,a_h^k\big|\text{traj}\!)\!}{\gapmin\!}
	\nonumber\\
	&\qquad+
	\sum_{\text{traj}\in\overline{\cE_{\mathrm{conc}}}}\!
	\prob\!\left(\text{traj}\right)\!\cdot\!TH
	\label{tmp1}\\
	{\le}&\prob\!\left(\cE_{\mathrm{conc}}\right)
	\sum_{n=1}^{N} 2^n\gapmin C^{(n)}
	+\prob\!\left(\overline{\cE_{\mathrm{conc}}}\right)\cdot TH
	\label{tmp2}\\
	{\le}&\sum_{n=1}^{N}\!\mathcal{O}\!\left(\frac{H^6SA\iota}{2^n \gapmin}\right)\!+\!H\label{tmp3}\\
	\le&\mathcal{O}\!\left(\!\frac{H^6SA}{\gapmin}
	\log\!\left(\!{SAT}\right)\!\right).\nonumber
	\end{align}	
	Above, \eqref{tmp0} follows from the definition of expectation,
	\eqref{tmp1} is because Ineq~(\ref{ineq:gap-bounded-by-learning-error}) suggests that for trajectories inside $\cE_{\mathrm{conc}}$, gaps can be bounded by clipped learning errors; whereas for trajectories outside of $\cE_{\mathrm{conc}}$, sub-optimality gaps never exceed $H$.
	\eqref{tmp2} follows from adding an outer summation for state-action pairs over the $N$ disjoint subintervals, then bounding the estimation error in each subinterval by its maximum value times the number of steps it contains. 
	\eqref{tmp3}~comes from a sum of numbers in a geometric progression generated by Lemma~\ref{lemma:count-in-each-layer}, and the fact that $\prob\!\left(\overline{\cE_{\mathrm{conc}}}\right)\le \nicefrac{1}{T}$ from concentration Lemma~\ref{lemma: conc-event}.
	In the final step, we notice that $\iota=\log(SAT^2)=\cO\left(\log(SAT)\right)$.	
\qed


\paragraph{Proof of Lemma~\ref{lemma:count-in-each-layer}}
The proof of Lemma~\ref{lemma:count-in-each-layer} relies on a general lemma (Lemma~\ref{lemma:weighed-sum-learning-error}) characterizing a weighted sum of the estimation errors of $Q$-function.
 Then we choose a particular sequence of weights to prove Lemma~\ref{lemma:count-in-each-layer}.
We remark that this general idea has appeared in \cite{jin2018qlearning,dong2019qlearning,zhang2020optimal}.

Formally, we use the following definition.
\begin{definition}[$(C,w)$-Sequence (Definition 3 in \cite{dong2019qlearning})] 
	\label{def:c-w-sequence}
	A sequence $\left\{w_{k}\right\}_{\!k\ge1}$ is called a $(C,w)$-sequence if $0\!\le\!w_{k}\!\le\!w$ for all $k$ and $\sum_{k}\!w_{k}\le C$. 
\end{definition}


\begin{lemma}[Weighted Sum of Learning Errors]
	\label{lemma:weighed-sum-learning-error}
	On event $\cE_{\mathrm{conc}}$, {for every $h\in[H]$, if $\left\{w_{k}\right\}_{k\in[K]}$ is a $(C,w)$-sequence, 
		then:}
	\begin{eqnarray}
	\sum_{k=1}^K \!w_{k}\!
	\left(\!Q_h^k-Q_h^*\!\right)\!(\!x_h^k,a_h^k\!)
	\!\le\! ewS\!AH^2\!+\!10c\sqrt{\!ewS\!AC\!H^5\iota}.\nonumber
	\end{eqnarray}
\end{lemma}

Before presenting the proof of Lemma~\ref{lemma:weighed-sum-learning-error}, we refer the readers to \cite{jin2018qlearning} for Lemma~\ref{lemma:property-lr} below, which summarizes the properties of $\alpha_t^i$ that will be useful in our proof.
\begin{lemma}[Properties of $\alpha_t^i$]
	\label{lemma:property-lr}
	Let $\alpha_t\!=\!\frac{H+1}{H+t}$, $\alpha_t^0\!=\!\prod_{j=1}^{t}\!\left(1\!-\!\alpha_{j}\right)$ and $\alpha_t^i\!=\!\alpha_i\!\prod_{j=i+1}^{t}\!\left(1\!-\!\alpha_{j}\right)$ for $0<i\le t$.
	\begin{enumerate}[(i)]
		\item $\sum_{i=1}^{t}\alpha_t^i=1$ and $\alpha_t^0=0$ for every $t\ge1$, $\sum_{i=1}^{t}\alpha_t^i=0$ and $\alpha_t^0=1$ for $t=0$.
		\item $\sum_{t=i}^{\infty}\alpha_t^i=1+\frac{1}{H}$ for every $i\ge1$. 
	\end{enumerate}
\end{lemma}
\paragraph{Proof of Lemma~\ref{lemma:weighed-sum-learning-error}}

We will recursively bound the weighted sum of step $h$ by its next step $(h\!+\!1)$, and unroll $(H\!-\!h\!+\!1)$ times for the desired bound. As suggested by Lemma~\ref{lemma: conc-event}, upper bounds of learning error holds under $\cE_{\mathrm{conc}}$. Thus we have
\begin{align}
&\sum_{k=1}^K w_{k}
\left(Q_h^k-Q_h^*\right)\!\xahk{h}{k}\nonumber\\
\le& \sum_{k=1}^K\! w_{k}
\Big(\! H\!\alpha_{n_h^k}^0\!+\!\beta_{\!n_h^k}\!+\!\sum_{i=1}^{n_h^k}\alpha_{n_h^k\!}^i\!\left(\!V_{\!h+1}^{\!\tau_h\!(s,a,i)}\!-\!V_{h+1}^{\!*}\!\right)\!(\!x_{h+1\!}^{\!\tau_h\!(\!s,a,i\!)\!}\!)\!
\Big)\nonumber\\
= &\sum_{k=1}^K w_{k} H\alpha_{n_h^k}^0+\sum_{k=1}^K w_{k} \beta_{n_h^k}\nonumber\\
&+\sum_{k=1}^K w_k
\sum_{i=1}^{n_h^k}\!\alpha_{n_h^k\!}^i\!\left(\!V_{\!h+1}^{\!\tau_h\!(x_h^k\!,a_h^k,i)}\!-\!V_{h+1}^{\!*}\!\right)\!(x_{h+1\!}^{\!\tau_h\!(\!x_h^k\!,a_h^k,i\!)\!}) 
\label{ineq:recursion}. 
\end{align}
For the first term of \eqref{ineq:recursion}, $n_h^k\!=\!0$ at most once for every state-action pair, and we always have $w_k\!\le\! w$. Thus, 
\begin{align}
\sum_{k=1}^K w_{k}H\alpha_{n_h^k}^0=\sum_{k=1}^K w_{k} H\indict{n_h^k=0}\le wSAH.
\label{first term}
\end{align}

The second term of (\ref{ineq:recursion}) can be bounded by the following inequalities with respective reasons listed below:
\begin{align}
\sum_{k=1}^K w_{k} \beta_{n_h^k}&=
\sum_{s,a}\! \sum_{k=1\atop (s_h^k,a_h^k)=(s,a)}^{K}\!w_{k}\beta_{n_h^k}\nonumber\\
&= 4c\sqrt{\!H^3\iota}
\sum_{s,a}\!
\sum_{i=2}^{N_h^{\!K}\!(s,a)}\!
\frac{w_{\!\tau\!(\!s,a,i\!)}}{\sqrt{i-1}} 
\label{tmp11}\\
&{\le}
4c\sqrt{\!H^3\iota}
\sum_{s,a}
\sum_{i=1}^{\left\lceil{C_{s,a}}/{w}\right\rceil}
\frac{w}{\sqrt{i}}
\label{tmp12}\\
&{\le}10c\sqrt{H^3\iota}
\sum_{s,a}\sqrt{C_{s,a} w}\label{tmp13}\\
&{\le}10c\sqrt{SACwH^3\iota}.
\label{ineq:beta-sum}
\end{align}

Above, \eqref{tmp11} comes from prior definition $\beta_t\!=\!4c\sqrt{\!\frac{H^3\iota}{t}}$ when $t\ge 1$ and $\beta_0=0$. Note that $\tau_h(x, a, i)$ is the episode where $(x, a)$ is visited for the $i$-th time, so we always have $n_h^{\tau_h(x,a,i)} = i-1$. \eqref{tmp12} follows from a rearrangement inequality with $C_{\!s,a\!}$ defined as $C_{\!s,a\!}:=\sum\nolimits_{i=1}^{n_h^{\!K}\!(\!s,a\!)}\!w_{\tau\!(\!s,a,i\!)}$, where we always keep in mind that $0<w_{\tau\!(\!s,a,i\!)}\le w$. \eqref{tmp13} follows from the integral conversion of $\sum_i{1}/{\sqrt{i}}$, and \eqref{ineq:beta-sum} is true because of Cauchy-Schwartz inequality where
$\sum\nolimits_{\!s,a\!}\!C_{\!s,a\!}\!=\!\sum\nolimits_{k=1}^K w_k\!\le\!C$.

For the third term in Ineq (\ref{ineq:recursion}), we notice that $\Vxhk{h}{k}=\Qxahk{h}{k}$ due to greedy choice of actions and $V_{h}^*\!(x_{h}^k)\ge Q_{h+1}^*\!(x_{h+1}^k,a_{h+1}^k)$ by definition. 
Therefore
 $\left(V_{h}^k-V_{h}^*\right)\!(x_{h}^k)\le \left(Q_{h}^k-Q_{h}^*\right)\!(x_{h}^k,a_{h}^k)$.
Note that $\forall\! k\!\in\![K]$, the third term takes into account all the prior episodes $l\!<\!k$ where $\xahk{h}{k}\!=\!\xahk{h}{l}$, indicating that the learning error at step $l$ is only counted by subsequent episodes $k\!>\!l$ when the same $(s,a)$ is visited. 
Thus, we exchange the order of summation and obtain
\begin{align}
&\sum_{k=1}^K w_k\sum_{i=1}^{n_h^k}
\alpha_{n_h^k\!}^i\!\left(\!V_{\!h+1}^{\!\tau_h\!(x_h^k,a_h^k,i)}\!-\!V_{h+1}^{\!*}\!\right)\!(x_{h+1\!}^{\!\tau_h\!(\!x_h^k,a_h^k,i\!)\!})\nonumber\\
=&\sum_{l=1}^{K}
\left(V_{\!h+1}^{l}\!-\!V_{h+1}^{\!*}\right)\!(x_{h+1\!}^{l})
\sum_{j= n_h^l+1}^{N_{\!h}^{\!K}\!\xahk{h}{l}}\!
w_{\tau_h\!(x_h^l,a_h^l,j)} \alpha_{j}^{\!n_h^l+1}\nonumber\\
\le&\sum_{l=1}^{K}
\left(\!Q_{\!h\!+\!1}^{l}\!-\!Q_{h\!+\!1}^{*}\!\right)\!(x_{h\!+\!1\!}^{l}, a_{h\!+\!1\!}^{l})
\sum_{j= n_h^l\!+1}^{N_{\!h}^{\!K}\!\xahk{h}{l}}\!
w_{\!\tau_h\!(x_h^l,a_h^l,j\!)}\! \alpha_{j}^{\!n_h^l\!+1}.\nonumber
\end{align}

Then for $l\in[K]$ we
let $\widetilde{w}_l=
\sum\limits_{j= n_h^l+1}^{N_{\!h}^{\!K}\!\xahk{h}{l}}\!
w_{\tau_h\!(x_h^l,a_h^l,j)} \alpha_{j}^{\!n_h^l+1}$ and further simplify the above equation to be
\begin{align}
&\sum_{k=1}^K w_k\sum_{i=1}^{n_h^k}
\alpha_{n_h^k\!}^i\!\left(\!V_{\!h+1}^{\!\tau_h\!(s,a,i)}\!-\!V_{h+1}^{\!*}\!\right)\!(x_{h+1\!}^{\!\tau_h\!(\!s,a,i\!)\!})\nonumber\\
\le&\sum_{l=1}^{K}\widetilde{w}_l
\left(Q_{\!h+1}^{l}\!-\!Q_{h+1}^{\!*}\right)\!(x_{h+1\!}^{l}, a_{h+1\!}^{l}).\label{third term}
\end{align}
Next, we use Lemma~\ref{lemma:property-lr} to verify that $\left\{\widetilde{w}_l\right\}_{l\in[K]}$ is a $\left(C,\left(\!1\!+\!\frac{1}{H}\!\right)\!w\right)$-sequence:
\begin{align}
\widetilde{w}_l&\le w\!\sum_{j= n_h^l\!+\!1}^{N_{\!h}^{\!K}\!\xahk{h}{l}}\alpha_j^{n_h^l\!+\!1}\le w\! \sum_{j\ge n_h^l\!+\!1} \!\alpha_j^{n_h^l\!+\!1}
\!\le\!  \left(\!1\!+\!\frac1H\!\right)\!w,\nonumber\\
\sum_{l=1}^{K}\!\widetilde{w}_{l}
&=\sum_{l=1}^K\sum\limits_{j= n_h^l+1}^{N_{\!h}^{\!K}\!\xahk{h}{l}}\!
w_{\tau_h\!\left(\!x_h^l,a_h^l,j\!\right)} \alpha_{j}^{\!n_h^l+1}\nonumber\\
&=\sum\limits_{k=1}^K\!w_{k}\sum\limits_{t=1}^ {n_h^k}
\alpha_{n_h^k}^t=\sum\limits_{k=1}^K\!w_{k}\le C.
\label{ineq:w_{h+1}}
\end{align}
Plugging the upper bounds of three separate terms in \eqref{first term}, \eqref{ineq:beta-sum} and \eqref{third term} back into Ineq~\eqref{ineq:recursion} gives us
\begin{align}
\sum_{k=1}^K \!w_{k}\!&
\left(\!Q_h^k\!-\!Q_h^*\!\right)\!\xahk{h}{k}
\!\le\! wSAH\!+\!10c\sqrt{SACwH^3\iota}\nonumber\\
&\!+\!\sum_{l=1}^K\!\widetilde{w}_{l}\!
\left(\!Q_{h\!+\!1}^l\!-\!Q_{h\!+\!1}^*\!\right)\!\xahk{h+1}{l},
\end{align}
where the third term is a weighted sum of learning errors of the same format, but taken at level $h+1$. In addition, it has weights $\left\{\!\widetilde{w}_{l}\right\}_{\!l\in[\!K\!]\!}$ being a  $\left(\!C,\left(\!1\!+\!{1}\!/\!{H}\!\right)\!w\!\right)$-sequence. Therefore, the above analysis will also yield
\begin{align}
&\sum_{l=1}^K \!\widetilde{w}_{l}\!
\left(\!Q_{h+1}^l\!-\!Q_{h+1}^*\!\right)\!\xahk{h}{l}
\!\le\! \left(\!1\!+\!\frac{1}{H}\!\right)\!wSAH\!\nonumber\\
&+\!10c\sqrt{SAC\!\left(\!1\!+\!\frac{1}{H}\!\right)\!wH^3\iota}
+
\big[\text{weighted sum at ($h\!+\!2$)}\big].\nonumber
\end{align}
Recursing this argument for $h\!+\!1,h\!+\!2,\!\cdots\!,\!H$ gives us
\begin{align}
&\sum_{k=1}^K \!w_{k,h}\!
\left(\!Q_h^k\!-\!Q_h^*\!\right)\!\xahk{h}{k}\nonumber\\
\le&\sum_{h'=0}^{H-h}\left(\!SAH\left(\!1\!+\!{1}\!/\!{H}\!\right)^{h'\!}\!w\!+\!10c\sqrt{\!SAC\left(\!1\!+\!{1}\!/\!{H}\!\right)^{h'}\!wH^3\iota}\!\right)\nonumber\\
\le& H\left(SAHew\!+\!10c\!\sqrt{SACewH^3\iota}\right).
\end{align}
which is the desired conclusion.

With Lemma~\ref{lemma:weighed-sum-learning-error}, we can easily prove Lemma~\ref{lemma:count-in-each-layer} by choosing a particular $(C,w)$-sequence.
\paragraph{Proof of Lemma~\ref{lemma:count-in-each-layer}}
For every $n\!\in\![N]$, $h\!\in\![H]$, let 
\begin{align*}
w_{k}^{\!(\!n\!,h\!)}&:=\indict{\left(\!Q_h^k-Q_h^*\!\right)\!
	(x_h^k,a_h^k)\!\in\!\left[2^{n\!-\!1\!}\gapmin,2^n \gapmin\right) }\!,\\
C^{\!(\!n\!,h\!)}&:=\sum_{k=1}^K w_{k}^{(n,h)}\\
&\!=\!\Bigg|\!\Big\{\!k\!:\!\left(\!Q_h^k-Q_h^*\!\right)\!(\!x_h^k,a_h^k\!)\!\in\!\left[\!2^{n\!-\!1\!}\gapmin,2^{n} \gapmin\!\right)\!\Big\}\!\Bigg|.
\end{align*}
By definition, $\forall h\!\in\![H]$ and $n\!\in\![N]$, $\left\{\!w_{k}^{\!(\!n\!,h\!)\!}\!\right\}_{\!k\in[K]}$ is a $(C^{\!(\!n\!,h\!)} ,1)$-sequence. 
Now we consider bounding 
$\sum_{k=1}^K\! w_{k}^{\!(\!n\!,h\!)}\!
\left(\!Q_h^k\!-\!Q_h^*\!\right)\!\xahk{h}{k}$ from both sides.
On the one hand, by Lemma~\ref{lemma:weighed-sum-learning-error},
\begin{align*}
&\sum_{k=1}^K\! w_{k}^{\!(\!n\!,h\!)}\!
\left(\!Q_h^k\!-\!Q_h^*\!\right)\!\xahk{h}{k}
\!\le\!
 eS\!A\!H^2\!+\!10c\!\sqrt{\!e\!S\!AC^{\!(\!n\!,h\!)}\!H^5\!\iota}.
\end{align*}
On the other hand, according to the definition of $w_k^{\!(\!n\!,h\!)}$,

\begin{align*}
\sum_{k=1}^K\! w_{k}^{\!(\!n\!,h\!)}
\!\left(\!Q_h^k\!-\!Q_h^*\!\right)\!\xahk{h}{k}
\!\ge\! \left(2^{n-1} \gapmin\right)\!\cdot \!C^{\!(\!n\!,h\!)}.
\end{align*}
Combining these two sides,
we obtain the following inequality of $C^{\!(\!n\!,h\!)}$: 
\begin{align}
\left(2^{n-1} \gapmin\right)& C^{\!(\!n\!,h\!)}\le eSAH^2+10c\sqrt{eSAC^{\!(\!n\!,h\!)}H^5\iota}\nonumber\\
\Rightarrow&\  C^{\!(\!n\!,h\!)}\le\mathcal{O}\left(\frac{H^5SA\iota}{4^n \gapmin^2}
\right).\nonumber
\end{align}
Finally, we observe that
\begin{align}
C^{(n)}=\sum_{h=1}^{H}C^{\!(\!n\!,h\!)}
\le\mathcal{O}\left(\frac{H^6SA\iota}{4^n \gapmin^2}\right),\nonumber
\end{align}
which is exactly the statement of Lemma~\ref{lemma:count-in-each-layer}.
\qed

\section{Conclusion and Future Directions}
\label{sec:conclusion}
This paper gives the first logarithmic regret bounds for $Q$-learning in both finite-horizon and discounted tabular MDPs. Below we list some future directions that we believe are worth exploring.

\paragraph{$H$ dependence}
The dependency on $H$ in our regret bound for episodic RL is $H^6$, which we believe is suboptimal.
As discussed in \cite{max2019nonasymptotic}, improving the $H$ dependence is often a challenging task.
Recently, \citet{zhang2020optimal} showed a model-free algorithm can achieve near-optimal regret in the worst case using the idea of reference value function.
It would be interesting to apply this idea to improve the $H$ dependence in our logarithmic regret bound.


\paragraph{Function Approximation}
Lastly, we note that recently researchers found the sub-optimality gap assumption is crucial for dealing with large state-space RL problems where function approximation is needed.
\citet{du2019q} presented an algorithm that enjoys polynomial sample complexity if there is a sub-optimality gap and the environment satisfies a low-variance assumption.
\citet{du2019good,du2020agnostic} further showed this assumption is necessary in certain settings.
There is another line of works putting certain low-rank assumptions on MDPs~\citep{krishnamurthy2016pac,jiang2017contextual,dann2018oracle,du2019provably,sun2018model,misra2019kinematic}.
It would be interesting to extend our analysis to these settings and obtain logarithmic regret bounds.

\bibliographystyle{plainnat}
\bibliography{ref}

\onecolumn

\section{Algorithm for Discounted MDP}
\label{sec:discounted_proof}
The pseudocode is listed in Algorithm~\ref{algo:infinite ucb-q}.
We acknowledge that Algorithm~\ref{algo:infinite ucb-q} relies on knowing a lower bound on $\gapmin$, and we leave it an open problem to develop a parameter-free algorithm. 

\begin{algorithm}[tb]
	\caption{{Infinite Q-learning with UCB-Hoeffding}\label{algo:infinite ucb-q}}
	\begin{algorithmic}[1]
		\State \textbf{Initialized:} $Q(x,a)\gets \frac{1}{1-\gamma}$ and $N(x,a)\gets0$ for all $(x,a)\in\states\times\actions$.
		\State \textbf{Define}   $\iota(k)\gets\log\left({ SAT(k+1)(k+2)}\right)$, $H\gets\frac{\ln\left(\nicefrac{2}{\discount\gapmin}\right)}{\ln\left(\nicefrac{1}{\gamma}\right)}$, $\alpha_k=\frac{H+1}{H+k}$.
		\For{step $t \in [T]$}\label{algo:for-loop}
		\State Take action $a_t\gets\argmax_{a'}Q\!\left(x_t,a'\right)$,  observe $x_{t+1}$.
		\State $k=\Nxa{t}\gets \Nxa{t}+1$,
		\State $b_k\gets \frac{c_2}{1-\gamma}\sqrt{H \iota(k)/k}$, \Comment{$c_2$ is a constant that can be set to $4\sqrt{2}$.}
		\State $\Vhatx{t+1}\gets\max_{a'\in \actions}\hQ\!\left(x_{t+1},a'\right)$,
		\State $\Qxa{t}\gets\left(1-\alpha_k\right)\Qxa{t}+\alpha_k\left[\rxa{t}+b_k+\gamma\Vhatx{t+1}\right]$,
		\State $\Qhatxa{t}\gets\min\left\{\Qhatxa{t},\Qxa{t}\right\}$.
		\EndFor
	\end{algorithmic}
\end{algorithm}

\section{Proofs for Discounted MDP}
\paragraph{Notations}
Let $Q^t(s,a),\widehat{Q}^t(s,a),V^t(s), \widehat{V}^t(s),N^t(s,a)$ denote the value of $Q(s,a),\widehat{Q}(s,a),V(s), \widehat{V}(s),N(s,a)$ right before the $t$-th step, respectively.
Let $\tau(s,a,i):=\max\left\{t:N^t(s,a)=i-1\right\}$ be the step $t$ at which $\xahk{}{t}=(x,a)$ for the $i$-th time. We will abbreviate $\Nxahk{}{t}$ for $n^t$ when no confusion can arise. $\alpha_{t}^i$ is defined same as that in the finite-horizon episodic setting.
\paragraph{Proof of Theorem~\ref{thm:discounted}}
We shall decompose the regret of each step as the expected sum of discounted gaps using the exact same argument as Eq~(\ref{eq:episodic-regret-decomp}), where the expect runs over all the possible infinite-length trajectories\footnote{
	For the convenience of analysis, when proving the upper bound we  remove the constraint $t\in[T]$ in the for-loop of line~\ref{algo:for-loop}. Instead, we allow the algorithm to take as many steps as we need, even yielding infinite-length trajectories.} taken by Algorithm~\ref{algo:infinite ucb-q}:
\begin{eqnarray}
\left(V^{*}-V^{\pi_t}\right)\!\left(s_t\right)=
\expect{\sum_{h=0}^{\infty}
	\gamma^h
	\gap\!\xah{t+h}\Bigg|
	a_{t+h}=\pi_{t+h}\left(s_{t+h}\right)}.
\end{eqnarray}
Based on this expression, the expected total regret over first $T$ steps can be rewritten as
\begin{eqnarray}
\expect{\regret(T)}&=&
\expect{\sum_{t=1}^T \left(V^*-V^{\pi_t}\right)(x_t)}
=\expect{\sum_{t=1}^T\expect{\sum_{h=0}^{\infty}
		\gamma^h\gap\!\xah{t+h}}}\nonumber\\
&=&\expect{\sum_{t=1}^{T}\sum_{h'=t}^{\infty}
	\gamma^{h'-t}\gap\!\xah{h'}}
\label{eq:disvounted-regret-decomp}
\end{eqnarray}
Our next lemma is borrowed from \citet{dong2019qlearning}, which shows that Algorithm~\ref{algo:infinite ucb-q} satisfies optimism and bounded learning error with high probability. By abuse of notation, we still use $\cE_{\mathrm{conc}}$ to denote the successful concentration event in this setting. Recall that Algorithm~\ref{algo:infinite ucb-q} specifies  $\iota(t)=\log\left(SAT(t+1)(t+2)\right)$ and $\beta_t=\frac{c_3}{1-\gamma}\sqrt{\frac{H\iota(t)}{ t}}$.
\begin{lemma}[Bounded Learning Error]
	\label{lemma:infinite-conc-event}
	Under Algorithm~\ref{algo:infinite ucb-q}, event $\cE_{\mathrm{conc}}$ occurs
	w.p. at least $1-\frac{1}{2T}$:
	\begin{eqnarray}
	\cE_{\mathrm{conc}}&:=&\Bigg\{
	\forall (x,\!a,\!t)\in\states\times\actions\times\bN_+\!
	:
	0\le\left(\!\hQ^t\!-\!Q^*\!\right)\!(\!x,a\!)
	\le\left(\!Q^t\!-\!Q^*\!\right)\!(x,a)\nonumber\\
	&&\qquad\qquad\qquad\qquad\qquad \le\frac{\alpha_{n^t}^0}{1-\gamma}\!+\!\sum_{i=1}^{n^t}\gamma\alpha_{n^{\!t}}^i\!\left(\!\hV^{\!\tau(x,a,i)}\!-\!V^*\!\right)\!\left(\!x_{\!\tau(x,a,i)}\!\right)\!+\!\beta_{n^{\!t}}\Bigg\} .\nonumber
	\end{eqnarray}
\end{lemma}

Then we proceed to present an analog of Lemma~\ref{lemma:weighed-sum-learning-error} that bounds the weighted sum of learning error in the discounted setting.

\begin{lemma}[Weighted Sum of Learning Errors]
	\label{lemma:infinite-weighted-sum-Q}
	Under $\cE_{\mathrm{conc}}$, for every $(\!C,\!w\!)$-sequence $\left\{w_t\right\}_{t\ge1}$, the following holds.
	\begin{eqnarray}
	\sum_{t\ge1}w_t\left(\hQ^t-Q^*\right)\!\xah{t}&\le&
	\frac{\gamma^H C}{1-\gamma}+\cO\left(
	\frac{\sqrt{wSAHC\iota(C)}+wSA}{\discount^2}
	\right)
	\end{eqnarray}
\end{lemma}
\begin{proof}
	Recall that Lemma~\ref{lemma:infinite-conc-event} bounds the learning error $\left(\!\hQ^t-Q^*\!\right)\!\xah{t}$ on $\cE_{\mathrm{conc}}$. Thus we have:
	\begin{align}
	\sum_{t\ge1} w_t\frac{\alpha_{n^{\!t}}^0}{1-\gamma}&\le
	\sum_{t\ge1}\indict{n^{\!t}=0}\cdot\frac{w}{1-\gamma}=\frac{SAw}{1-\gamma};&
	\label{ineq:first}\\
	\sum_{t\ge1}w_t\beta_{n^{\!t}}
	&=\sum_{s,a}\! \sum_{i}
	\!w_{\tau\!(s,a,i)}\beta_{i}
	\!=\!\frac{c_3\!\sqrt{\!H}}{1-\gamma}\sum_{s,a}\! \sum_{i}
	\!w_{\!\tau\!(\!s,a,i)}\!\sqrt{\!\frac{\iota(i)}{i}}&\nonumber\\
	&\le\frac{c_3\!\sqrt{\!H}}{1-\gamma}\sum_{s,a}
	\!\sum_{i=1}^{C_{\!s,a}\!/\!w}
	\!w\sqrt{\frac{\iota(C)}{i}}
	\le\frac{2c_3\!\sqrt{\!H}}{1-\gamma}\sum_{s,a}
	\!\sqrt{C_{\!s,a}w{\iota(C)}} &
	\left(\!C_{\!s,a}\!:=\!\sum\nolimits_{\!i\ge1}
	\!w_{\tau\!(\!s,a,i)}\!\right)\nonumber\\
	&\le\frac{2c_3\!}{1-\gamma}\sqrt{SAHCw{\iota(C)}}; &\text{(Cauchy-Schwartz inequality)}
	\label{ineq:second}\\
	\nonumber\end{align}
	Moreover,
	\begin{align}
	&\sum_{t\ge1}w_t\sum_{i=1}^{n^t}\gamma       \alpha_{n^{\!t}}^i
	\!\left(\!\hV^{\!\tau(x,a,i)}\!-\!V^*\!\right)\!
	\left(\!x_{\!\tau(x,a,i)}\!\right)&\nonumber\\
	=&\gamma\sum_{t\ge1}\left(\!\hV^t-V^*\!\right)\!(x_{t+1})
	\sum_{i\ge n^t+1}\!
	w_{\tau\!(x_t,a_t,i)}\alpha^{n^t+1}_{i}&\nonumber\\
	=&\gamma\sum_{t\ge1}\left(\hV^{t+1}-V^*\right)\!(x_{t+1})\sum_{i\ge n^t+1}\!
	w_{\tau\!(x_t,a_t,i)}\alpha^{n^t+1}_{i}
	+\gamma\sum_{t\ge1} \sum_{i\ge n^t+1}\!
	w_{\tau\!(x_t,a_t,i)}\alpha^{n^t+1}_{i}\!\left(\!\hV^{t}-\hV^{t+1}\!\right)\!( x_t).
	\label{ineq:v-decrease}
	\end{align}
	We let
	\begin{align*}
	\widetilde{w}_{t+1}:=\sum_{i\ge n^t+1}\!
	w_{\tau\!(x_t,a_t,i)}\alpha^{\!n^t+1}_{i},
	\end{align*}
	and further simplify \eqref{ineq:v-decrease} to be
	\begin{align}
	\gamma\sum_{t\ge2}\widetilde{w}_t\left(\hV^{t}-V^*\right)\!(x_t)
	+\gamma\sum_{t\ge1} \widetilde{w}_{t+1}\!\left(\!\hV^{t}-\hV^{t+1}\!\right)\!( x_t).
	\label{ineq:simplified 23}
	\end{align}

	For the first term of \eqref{ineq:simplified 23}, 
	we claim that $\left\{\widetilde{w}_t\right\}_{\!t\ge2}$ is a  $\left(\!C,(1+1\!/\!H)w\!\right)$-sequence. This can be verified by a similar argument to Ineq~(\ref{ineq:w_{h+1}}). We also have
	\begin{align*}
	\left(\hV^{t}-V^*\right)\!(x_t)&=\hV^t(x_t)-V^*(x_t)=\hQ^t(x_t,a_t)-V^*(x_t)\le\hQ^t(x_t,a_t)-Q^*(x_t,a_t)=\left(\hQ^t-Q^*\right)(x_t,a_t).
	\end{align*}
	Therefore, the first term can be upper bounded by
	\begin{align}
	\gamma\sum_{t\ge2}\widetilde{w}_t\left(\hV^{t}-V^*\right)\!(x_t)\le\gamma\sum_{t\ge 2}\widetilde{w}_t\left(\hQ^t-Q^*\right)(x_t,a_t).
	\label{ineq:discounted first term}
	\end{align}
	For the second term of \eqref{ineq:simplified 23}, we have the following observation:
	
	\begin{align}
	\gamma\sum_{t\ge1} \widetilde{w}_{t+1}\!\left(\!\hV^{t}-\hV^{t+1}\!\right)\!(x_t)
	&\le\gamma(1+1\!/\!H)w\sum_s\sum_{t\ge1}\left(\!\hV^t-\hV^{t+1}\!\right)\!(s)\nonumber\\
	&\le\gamma(1+1\!/\!H)w\sum_s\hV^1\!(s)
	\le \frac{\gamma (1+1\!/\!H)wS}{1-\gamma}.\label{ineq:discounted second term}
	\end{align}.
	
	Plugging \eqref{ineq:discounted first term} and \eqref{ineq:discounted second term} back into \eqref{ineq:simplified 23}, we obtain
	
	\begin{align}
	\sum_{t\ge1}w_t\sum_{i=1}^{n^t}\gamma       \alpha_{n^{\!t}}^i
	\!\left(\!\hV^{\!\tau(x,a,i)}\!-\!V^*\!\right)\!
	\left(\!x_{\!\tau(x,a,i)}\!\right)	\le&\gamma\sum_{t\ge2}\widetilde{w}_t\left(\hQ^{t}-Q^*\right)\!\xah{t}
		+\frac{\gamma (1+1\!/\!H)wS}{1-\gamma}. \label{ineq: discounted big third term}
	\end{align}
	Finally, combining \eqref{ineq:first}, \eqref{ineq:second} and \eqref{ineq: discounted big third term} and Lemma~\ref{lemma:infinite-conc-event}, we conclude that
	\begin{align}
	&\sum_{t\ge1}\!w_t\left(\!\hQ^t-Q^*\!\right)\!\xah{t}&\nonumber\\
	\le&\sum_{t\ge1}w_t
	\left(\frac{\alpha_{n^{\!t}}^0}{1-\gamma}
	+\beta_{n^{\!t}}+
	\gamma\sum_{i=1}^{n^{\!t}} \alpha_{n^{\!t}}^i
	\left(\hV^{\tau\!(s,a,i)}-V^*\right)\!\left(x_{\tau\!(s,a,i)+1}\right)
	\right)
	&\text{(Lemma \ref{lemma:infinite-conc-event})}\nonumber\\
	\le&\frac{SAw}{1-\gamma}
	+\frac{2c_3\!}{1-\gamma}\sqrt{SAHCw{\iota(C)}}
	+\frac{\gamma (1+1\!/\!H)wS}{1-\gamma}
	+\gamma\sum_{t\ge2}\widetilde{w}_t\left(\hQ^{t}-Q^*\right)\!\xah{t}&\label{ineq:inf-recursion}
	\end{align}
	Note that the last term in Ineq (\ref{ineq:inf-recursion}) is another weighted sum of learning errors starting from step 2, where the weights form a $\left(\!C,(1+1\!/\!H)w\!\right)$-sequence. We can therefore repeat this unrolling argument for $H$ times. Our choice of $H$ in Algorithm~\ref{algo:infinite ucb-q} guarantees not only the bounded blow-up factor of weights, but also sufficiently small contribution of learning error after step $H$. In particular, we define a family of weights: when $h=0$, $\left\{w_t^{(h)}\right\}_{t\ge h+1}=\{w_t\}_{t\ge1}$ is a $(C,w)$ sequence; $\forall h\in[H]$ $\left\{w_t^{(h)}\right\}_{t\ge h+1}$ is a $\left(\!C,(1+1\!/\!H)^hw\!\le\! ew\!\right)$ sequence. Note that our previous definition of $\widetilde{w}$ is exactly $w_t^{(1)}$.
	\begin{eqnarray}
	&&\sum_{t\ge1}\!w_t\left(\!\hQ^t-Q^*\!\right)\!\xah{t}\nonumber\\
	&\le&\sum_{h=0}^{H}\gamma^h\cO\!\left(
	\frac{\sqrt{(1+1\!/\!H)^hwSAHC\iota(C)}+(1+1\!/\!H)^hwSA}{1-\gamma}
	\right)+\gamma^H\sum_{t\ge H+1}w_t^{(H)}\left(\hQ^{t}-Q^*\right)\!\xah{t}\nonumber\\
	&\le&\cO\!\left(
	\frac{\sqrt{wSAHC\iota(C)}+wSA}{\discount^2}
	\right)+\frac{\gamma^H}{1-\gamma}\sum_{t\ge H+1}w_t^{(H)}.
	\end{eqnarray} 
	Using the fact that the weights after $H$ unrolling $\left\{w_t^{(H)}\right\}_{t\ge H+1}$ is a $\left(\!C,(1+1\!/\!H)^Hw\!\le\! ew\!\right)$-sequence
	completes the proof.
\end{proof}

Note that we have clarified in Ineq~(\ref{ineq:gap-bounded-by-learning-error}) that on $\cE_{\mathrm{conc}}$ where optimism holds, sub-optimality gaps can be bounded by clipped learning error of $Q$-function. Again we divide its range $\left[\gapmin,\frac{1}{1-\gamma}\right]$ into disjoint subintervals and bound the sum inside each subinterval independently. 
\begin{lemma}
	\label{lemma:infinite-count-in-each-layer} Let $N=\left\lceil \log_2\left(\nicefrac{1}{\gapmin\discount}\right)\right\rceil$. On $\cE_{\mathrm{conc}}$,
	for every $n\in\left[N\right]$,
	\begin{align}
	C^{(n)}:=&\Bigg|\Big\{\!t\in\bN_+:\left(\hQ^t-Q^*\right)\!\xah{t}\in\left[2^{n-1}\gapmin,2^n \gapmin\right)\!\Big\}\Bigg|\nonumber\\
	&\qquad\qquad\qquad
	\le\mathcal{O}\left(\frac{SA}{4^n\gapmin^2\discount^5}
	\ln\left(
	\frac{SAT}{\discount\gapmin}
	\right)
	\right).\nonumber
	\end{align} 
\end{lemma}
Again, based on Lemma~\ref{lemma:infinite-conc-event}, we prove Lemma~\ref{lemma:infinite-count-in-each-layer} by choosing a particular sequence of weights.
\begin{proof}
	For every $n\in[N]$, let 
	\begin{eqnarray}
	w_{t}^{(n)}&:=&\indict{\left(\hQ^t-Q^*\right)\xah{t}\in\left[2^{n-1}\gapmin,2^n \gapmin\right) },
	\end{eqnarray}
	then $C^{(n)}=\sum_{t=1}^{\infty} w_{t}^{(n)}$ and $\left\{\!w_{t}^{(n)}\!\right\}_{t\ge1}$ is a $(C^{(n)} ,\!1)$-sequence.  According to Lemma~\ref{lemma:infinite-weighted-sum-Q},
	\begin{align}
	\left(2^{n-1} \gapmin\right)\!\cdot \!C^{(n)}
	\le&\sum_{t\ge1} w_{t}^{(n)}
	\left(\hQ^t-Q^*\right)\!\xah{t}&\nonumber\\
	\le&\frac{\gamma^H C^{(n)}}{1-\gamma} +\cO\left(
	\frac{\sqrt{SAHC^{\!(n)}\iota\!\left(\!C^{\!(n)}\!\right)}+SA}{\discount^2}
	\right)&\nonumber\\
	=&\frac{\gapmin}{2} C^{\!(n)}+\cO\!\left(
	\frac{\sqrt{SAHC^{\!(n)}\iota\!\left(\!C^{\!(n)}\!\right)}+SA}{\discount^2}
	\right).&\left(H=\frac{\ln\left(\frac{2}{\gapmin\discount}\right)}{\ln\left(\nicefrac{1}{\gamma}\right)}\right) \label{ineq:inf-to-solve}\nonumber
	\end{align}
	Now we proceed to solve the above inequality for $C^{(n)}$. For simplicity, let $\delta=2^{n-2}\gapmin$ and $C^{(n)}=SAC'$. Then we have the following:
	\begin{eqnarray}
	\delta\cdot SAC'
	&\le&\left(2^{n-1}-\frac{1}{2}\right)\gapmin C^{\!(n)}\le \cO\!\left(SA
	\frac{\sqrt{HC'\iota\!\left(\!SAC'\!\right)}+1}{\discount^2}
	\right),\nonumber\\
	\delta C'&\overset{\circled{1}}{\le}&\cO\left(\frac{\sqrt{C'}}{\discount^{\nicefrac{5}{2}}}\sqrt{\ln{(SATC')}\ln\frac{1}{\gapmin\discount}}\ \right),\nonumber\\
	C'&\le&\cO\left(
	\frac{1}{\delta^2\discount^5}\ln\frac{1}{\gapmin\discount}\ln{(SATC')}
	\right),\label{ineq:C'-final}
	\end{eqnarray}
	where $\circled{1}$ comes from the definition $H=\frac{\ln\left(\nicefrac{2}{\gapmin\discount}\right)}{\ln\left(\nicefrac{1}{\gamma}\right)}$.
	Solving Ineq~(\ref{ineq:C'-final}) yields \[
	C'\le \frac{1}{\delta^2\discount^5} \ln\left(\frac{SAT}{\gapmin\discount}\right).
	\]
	Finally, substituting $C^{(n)}=SAC'$ and  $\delta=2^{n-2}\gapmin$ yields the desired formula.
\end{proof}

\begin{proof}[Proof of Theorem~\ref{thm:discounted}]
	We continue the calculation based on the regret decomposition in Eq~(\ref{eq:disvounted-regret-decomp}). For every infinite-length trajectory $\text{traj}\in\cE_{\mathrm{conc}}$,
	\begin{align}
	\sum_{t=1}^{T}\sum_{h'=t}^{\infty}
	\gamma^{h'-t}\gap\!(x_{h'},a_{h'}|\text{traj})
	&\overset{\circled{2}}{=}\sum_{h=1}^{\infty}
	\gap\!\xah{h}
	\sum_{t=1}^{\min\{T,h\}}
	\gamma^t\le\frac{1}{1-\gamma}\sum_{h=1}^{\infty}
	\gap\!(x_{h},a_{h}|\text{traj})\nonumber\\
	&\overset{\circled{3}}{\le} \frac{1}{1-\gamma}\sum_{t\ge1}
	\clip{\left(\hQ^t-Q^*\!\right)\!(x_{t},a_{t}|\text{traj})}{\gapmin}\nonumber\\
	&\overset{\circled{4}}{\le}\frac{1}{1-\gamma}\sum_{n=1}^{N} 2^n\gapmin C^{(n)}\nonumber\\
	&\overset{\circled{5}}{\le}\mathcal{O}\left(\frac{SA}{\gapmin\discount^6}\ln\left(
	\frac{SA}{p\epsilon\discount\gapmin}
	\right)
	\right).\label{ineq:inside conc}
	\end{align}
	For the above inequalities, $\circled{2}$ comes from an interchange of summations, $\circled{3}$ is by optimism of estimated $Q$-values, $\circled{4}$ is because we can add an outer summation over subintervals $n\in[N]$ and bound each of them by their maximum value times the number of steps inside. Finally, $\circled{5}$ follows directly from Lemma~\ref{lemma:infinite-count-in-each-layer}.
	
	On the other hand, for trajectories outside of $\cE_{\mathrm{conc}}$, since sub-optimality gaps are upper bounded by $\nicefrac{1}{1-\gamma}$, we have:
	\begin{align}
	\sum_{t=1}^{T}\sum_{h'=t}^{\infty}
	\gamma^{h'-t}\gap(x_{h'},a_{h'}|\text{traj})
	&\le \sum_{t=1}^{T}\sum_{h'=t}^{\infty}
	\gamma^{h'-t}\frac{1}{1-\gamma}
	\le\frac{T}{\discount^2}.\label{ineq:outside conc}
	\end{align}
	Therefore, combining Ineq (\ref{ineq:inside conc}) and (\ref{ineq:outside conc}) gives us
	\begin{align}
	\expect{\regret(T)}
	&=\expect{\sum_{t=1}^{T}\sum_{h'=t}^{\infty}
		\gamma^{h'-t}\gap\!\xah{h'}}\nonumber\\
	&\le
	\prob\!\left(\cE_{\mathrm{conc}}\right)\cdot
	\mathcal{O}\left(\frac{SA}{\gapmin\discount^6}\ln\left(
	\frac{SAT}{\discount\gapmin}
	\right)
	\right)+
	\prob\!\left(\overline{\cE_{\mathrm{conc}}}\right)\cdot \frac{T}{\discount^2}
	\nonumber\\
	&\le \mathcal{O}\left(\frac{SA}{\gapmin\discount^6}\ln\left(
	\frac{SAT}{\discount\gapmin}
	\right)
	\right),\label{ineq:infinite-regret}
	\end{align}
	where the last step is comes from $\prob\!\left(\overline{\cE_{\mathrm{conc}}}\right)\!\le\!{1}/{2T}$. Ineq~(\ref{ineq:infinite-regret}) is precisely the assertion of Theorem \ref{thm:discounted}.
\end{proof}

\section{Difficulty in Applying Optimistic Surplus}
\label{sec:opt_surplus}
The closest related work is by \citet{max2019nonasymptotic} who proved the logarithmic regret bound for a model-based algorithm.
\citet{max2019nonasymptotic} 
introduced a novel property characterizing optimistic algorithms, which is called \textit{optimistic surplus} and defined as
\begin{align}
E_{k,h}(x,a):=Q_h^k(x,a)-\left[r_h(x,a)+\trans_h(x,a)^{\!\mathsf{T}}V_{h+1}^k\right].\label{eqn:opt_surplus}
\end{align}
Under model-based algorithm with bonus term $b_h^k$, surplus can be decomposed as follows, where $\widehat{\trans}$ is the estimated transition probability:
\begin{align*}
E_{k,h}(x,a)=
\left(\widehat{\trans}_h^{\!\mathsf{T}}\!(x,a)-\trans_h^{\mathsf{T}}\!(x,a)\!\right){V}_{h+1}^*+\left(\widehat{\trans}_h^{\!\mathsf{T}}\!(x,a)-\trans_h^{\!\mathsf{T}}\!(x,a)\right)\left({V}_{h+1}^k-V_{h+1}^*\right)+b_h^k. 
\end{align*}
The analysis of model-based algorithms is to first bound the regret $\left(\!V^*-V^{\pi_k}\!\right)$ by a sum over surpluses that are clipped to zero whenever being smaller than some $\gap$-related quantities, then combine the concentration argument and properties of specially-designed bonus terms $b_h^k$ to provide high probability bound for surpluses.
However, for model-free algorithms, estimates of transition probabilities are no longer maintained, so $\widehat{\trans}_h$ is a one-hot vector reflecting only the current step's empirical sample drawn from the real next-state distribution. In this scenario, concentration argument of $\left(\widehat{\trans}-\trans\right)$ cannot give us $\log T$ regret.

Following the update rule of $Q$-learning with learning rate $\alpha_i$ and upper confidence bound $b_i$, the surplus becomes
\[
E_{k,h}(x,a)=\alpha_t^0 H+\left(\sum_{i=1}^{t}\alpha_t^i V_{h+1}^{k_i}(x_{h+1}^{k_i})-\trans_h(x,a)^{\!\mathsf{T}}V_{h+1}^k\right)+\sum_{i=1}^t \alpha_t^i b_i,
\]
in which $t=n_h^k(x,a)$ is the number of times $(x,a)$ has been visited, and $\alpha_t^i=\alpha_{i}\prod_{j=i+1}^t (1-\alpha_{j})$ is the equivalent weight associated with the $i$-th visit of pair $(x,a)$.
This indicates that the surplus of an episode is closely correlated with estimates of value functions during previous episodes. 
The correlation makes the analysis more difficult.
Therefore, we use a very different approach to analyze $Q$-learning in this paper.

\end{document}